\documentclass[12pt]{article}

\usepackage{graphicx,subfigure,amsmath,amssymb,amsfonts,bm,epsfig,epsf,url,dsfont}
\usepackage{times}
\usepackage{bbm}      
\usepackage{booktabs}
\usepackage{cases}
\usepackage{fullpage}
\usepackage[small,bf]{caption}
\usepackage{natbib}
\usepackage[top=1in,bottom=1in,left=1in,right=1in]{geometry}

\renewcommand{\phi}{\varphi}

\renewcommand{\P}{\mathbb{P}}
\newcommand{\E}{\mathbb{E}}

\newcommand{\LG}{\overline{\log}(K)}
\newcommand{\LGG}{\overline{\log}(M K)}

\newcommand{\cO}{\mathcal{O}}

\def\ds1{\mathds{1}}
\renewcommand{\epsilon}{\varepsilon}

\newcommand{\wh}{\widehat}
\newcommand{\argmax}{\mathop{\mathrm{argmax}}}

\renewcommand{\tilde}{\widetilde}

\newlength{\minipagewidth}
\newcommand{\bookbox}[1]{
\par\medskip\noindent
\framebox[\textwidth]{
\begin{minipage}{\minipagewidth}
{#1}
\end{minipage} } \par\medskip }

\newcommand{\beq}{\begin{equation}}
\newcommand{\eeq}{\end{equation}}

\newcommand{\beqa}{\begin{eqnarray}}
\newcommand{\eeqa}{\end{eqnarray}}

\newcommand{\beqan}{\begin{eqnarray*}}
\newcommand{\eeqan}{\end{eqnarray*}}

\def\ba#1\ea{\begin{align*}#1\end{align*}} 
\def\banum#1\eanum{\begin{align}#1\end{align}} 

\newtheorem{theorem}{Theorem}
\newcommand{\BlackBox}{\rule{1.5ex}{1.5ex}}  
\newenvironment{proof}{\par\noindent{\bf Proof\ }}{\hfill\BlackBox\\[2mm]}

\begin{document}

\title{Multiple Identifications in Multi-Armed Bandits}
\author{
S{\'e}bastien Bubeck \\
Department of Operations Research and Financial Engineering, \\
Princeton University \\
{\tt sbubeck@princeton.edu} \\ \\
Tengyao Wang \\
Department of Mathematics, \\
Princeton University \\\
{\tt tengyaow@princeton.edu}  \\ \\
Nitin Viswanathan \\
Department of Computer Science, \\
Princeton University \\
{\tt nviswana@princeton.edu} 
}

\date{\today}

\maketitle

\begin{abstract}
We study the problem of identifying the top $m$ arms in a multi-armed bandit game. Our proposed solution relies on a new algorithm based on successive rejects of the seemingly bad arms, and successive accepts of the good ones. This algorithmic contribution allows to tackle other multiple identifications settings that were previously out of reach. In particular we show that this idea of successive accepts and rejects applies to the multi-bandit best arm identification problem.
\end{abstract}

\section{Introduction} \label{sec:intro}
We are interested in the following situation: An agent faces $K$ unknown distributions, and he is allowed to do $n$ sequential evaluations of the form $(i,X)$ where $i \in \{1, \hdots, K\}$ is chosen by the agent and $X$ is a random variable drawn from the $i^{th}$ distribution and revealed to the agent. The goal of the agent after the $n$ evaluations is to identify a subset of the distributions (or \emph{arms} in the multi-armed bandit terminology) corresponding to some prespecified criterion. This setting was introduced in \cite{BMS09}, where the goal was to identify the distribution with maximal mean. Note that in this formulation of the problem the evaluation budget $n$ is fixed. Another possible formulation is the one of the PAC model studied in \cite{EMM02,MT04} where there is an accuracy of $\epsilon$ and a probability of correctness $\delta$ that are prespecified, and one wants to minimize the number of evaluations to attain this prespecified accuracy and probability of correctness. This latter formulation has a long history which goes back to the seminal work \cite{Bec54}. In this paper we focus on the fixed budget setting of \cite{BMS09}. For this fixed budget problem, \cite{ABM10} proposed a new analysis and an optimal algorithm (up to a logarithmic factor). In particular this work introduced a notion of \emph {best arm identification complexity}, and it was shown that this quantity, denoted $H$, characterizes the hardness of identifying the best distribution in a specific set of $K$ distributions. Intuitively, it was shown that the number of evaluations $n$ has to be $\Omega( H / \log K )$ to be able to find the best arm, and the algorithm SR (Successive Rejects) finds it with $\cO( H \log^2 K )$ evaluations. Furthermore in the latter paper the authors also suggested the open problem of generalizing the analysis and algorithms to the identification of the $m$ distributions with the top $m$ means. Our main contribution is to solve this open problem. We suggest a non-trivial extension of the complexity $H$, denoted $H^{\langle m \rangle}$, to the problem of identifying the top $m$ distributions, and we introduce a new algorithm, called SAR (Successive Accepts and Rejects), that requires only $\tilde{\cO}\left(H^{\langle m \rangle}\right)$\footnote{In the $m$-best arms identification problem we write $u_n = \tilde{\cO}(v_n)$ when $u_n = \cO(v_n)$ up to logarithmic factor in $K$} evaluations to find the top $m$ arms. We also propose a numerical comparison between SAR, SR and uniform sampling for the problem of finding the $m$ top arms. Interestingly the experiments show that SR performs badly for $m>1$, which shows that the tradeoffs involved in this generalized problem are fundamentally different from the ones for the single best arm identification.
\newline

As a by-product of our new analysis we are also able to solve an open problem of \cite{GGLB11}. In this paper the authors studied the setting where the agent faces $M$ distinct best arm identification problems. A multi-bandit identification complexity was introduced, that we denote $H^{[M]}$. On the contrary to the setting of single best arm identification, here the algorithm proposed in \cite{GGLB11} that needs of order of $H^{[M]}$ evaluations to find the best arm in each bandit requires to know the complexity $H^{[M]}$ to tune its parameters. Using our SAR machinery, we construct a parameter-free algorithm that identify the best arm in each bandit with $\tilde{\cO}\left(H^{[M]}\right)$\footnote{In the multi-bandit best arm identification problem we write $u_n = \tilde{\cO}(v_n)$ when $u_n = \cO(v_n)$ up to logarithmic factor in $M K$} evaluations.
\newline

Both the $m$-best arms identification and the multi-bandit best arm identification have numerous potential applications. We refer the interested reader to the previously cited papers for several examples.

\section{Problem setup} \label{sec:setup}
We adopt the terminology of multi-armed bandits. The agent faces $K$ arms and he has a budget of $n$ evaluations (or \emph{pulls}). To each arm $i \in \{1, \hdots, K\}$ there is an associated probability distribution $\nu_i$, supported\footnote{One can directly generalize the discussion to $\sigma$-subgaussian distributions.} on $[0,1]$. These distributions are unknown to the agent. The sequential evaluations protocol goes as follows: at each round $t = 1, \hdots, n$, the agent chooses an arm $I_t$, and observes a reward drawn from $\nu_{I_t}$ independently from the past given $I_t$. In the $m$-best arms identification problem, at the end of the $n$ evaluations, the agent selects $m$ arms denoted $J_1, \hdots, J_m$. The objective of the agent is that the set $\{J_1, \hdots, J_m\}$ corresponds to the set of arms with the $m$ highest mean rewards.
\newline

Denote by $\mu_1, \hdots, \mu_K$ the mean of the arms. In the following we assume that $\mu_1 > \hdots > \mu_K$. The ordering assumption comes without loss of generality, and the assumption that the means are all distinct is made for sake of notation (the complexity measures are slightly different if there is an ambiguity for the top $m$ means). We evaluate the performance of the agent's strategy by the probability of misidentification, that is
$$e_n = \P\left( \{J_1, \hdots, J_m\} \neq \{1, \hdots, m\} \right).$$
Finer measures of performance can be proposed, such as the simple regret $r_n = \sum_{i=1}^m (\mu_i - \E \mu_{J_i})$. However, as it was argued in \cite{ABM10}, for a first order analysis it is enough to focus on the quantity $e_n$. 
\newline

In the (single) best arm identification, \cite{ABM10} introduced the following complexity measures. Let $\Delta_i = \mu_1 - \mu_i$ for $i \neq 1$, $\Delta_1 = \mu_1 - \mu_2$, 
$$H_1=\sum_{i=1}^{K} \frac{1}{\Delta_i^{2}} \qquad \mbox{and} \qquad H_2 = \max_{i \in \{1,\hdots,K\}} i \Delta_{i}^{-2}.$$
It is easy to see that these two complexity measures are equivalent up to a logarithmic factor since we have (see \cite{ABM10})
  \begin{equation} \label{eq:hcd}
  H_2 \leq H_1 \leq \log(2K) H_2.
  \end{equation}
[Theorem 4, \cite{ABM10}] shows that the complexity $H_1$ represents the hardness of the best arm identification problem. However, as far as upper bounds are concerned, the quantity $H_2$ proved to be a useful surrogate for $H_1$. For the $m$-best arms identification problem we define the following gaps and the associated complexity measures:
\begin{eqnarray*}
\Delta_i^{\langle m \rangle} & = & \left\{ \begin{array}{ccc} \mu_i - \mu_{m+1} & \text{if} & i \leq m \\ \mu_m - \mu_{i} & \text{if} & i > m \end{array} \right. , \\
H_1^{\langle m \rangle} &  = & \sum_{i=1}^K \frac{1}{\left(\Delta_i^{\langle m \rangle} \right)^2} , \\
H_2^{\langle m \rangle} & = & \max_{i \in \{1,\hdots,K\}} i \left(\Delta_{(i)}^{\langle m \rangle}\right)^{-2} ,
\end{eqnarray*}
where the notation $(i) \in \{1, \hdots, K\}$ is defined such that $\Delta_{(1)}^{\langle m \rangle} \leq \hdots \leq \Delta_{(K)}^{\langle m \rangle}$.
We conjecture that a similar lower bound to [Theorem 4, \cite{ABM10}] with $H_1$ replaced by $H_1^{\langle m \rangle}$ holds true for the $m$-best arms identification problem. In this paper we shall prove an upper bound on $e_n$ that gets small when $n = \tilde{\cO} \left( H_2^{\langle m \rangle} \right)$ (recall that by \eqref{eq:hcd}, $\tilde{\cO} \left( H_2^{\langle m \rangle} \right) = \tilde{\cO} \left( H_1^{\langle m \rangle} \right)$). This result is derived in Section \ref{sec:mbest}, where we introduce our key algorithmic contribution, the SAR (Successive Accepts and Rejects) algorithm. We also present experiments for this setting in Section \ref{sec:exp}.
\newline

In Section \ref{sec:multi} we consider the framework of multi-bandit introduced in \cite{GGLB11}, where the agent faces $M$ distinct best arm identification problems. For sake of notation we assume that each problem $m \in \{1, \hdots, M\}$ has the same number of arms $K$. We also restrict our attention to the single best arm identification within each problem, but we could deal with $m$-best arms identification within each problem. We denote by $\nu_1(m), \hdots, \nu_K(m)$ the unknown distributions of the arms in problem $m$. We define similarly all the relevant quantities for each problem, that is $\mu_1(m) > \hdots > \mu_K(m), \Delta_1(m), \hdots, \Delta_K(m), H_1(m)$ and $H_2(m)$. Finally we denote by $(i,m)$ the arm $i$ in problem $m$. 
In the multi-bandit best arm identification, the forecaster performs $n$ sequential evaluations of the form $(I_t, m_t) \in \{1, \hdots, K\} \times \{1, \hdots, M\}$. At the end of the $n$ evaluations, the agent selects one arm for each problem, denoted $(J_1, 1), \hdots, (J_M, M)$. The objective of the agent is to find the arm with the highest mean reward in each problem, that is in this setting the probability of misidentification can be written as
$$e_n = \P( \exists m \in \{1, \hdots, M\} : J_m \neq 1 ).$$ 
Following \cite{GGLB11} we introduce the following complexity measure
$$H_1^{[M]}  = \sum_{m=1}^M H_1(m) .$$ 
Again we define a sort of weaker complexity measure by ordering the gaps. 
Let 
$$
	\Delta_{1}^{[M]} \leq \Delta_{2}^{[M]} \leq \cdots \leq \Delta_{MK}^{[M]}
$$ 
be a rearrangement of $\{\Delta_{i}(m) : 1\leq i\leq K, 1\leq m\leq M\}$ in ascending order, and let
$$
	H_2^{[M]} = \max_{k \in \{1, \hdots, M K\}} k \left( \Delta_{k}^{[M]} \right)^{-2}.
$$
We conjecture that a similar lower bound to [Theorem 4, \cite{ABM10}] with $H_1$ replaced by $H_1^{[M]}$ holds true for the multi-bandit best arm identification problem. In this paper we shall prove an upper bound on $e_n$ that gets small when $n = \tilde{O} \left(H_2^{[M]}\right)$ (recall that by \eqref{eq:hcd}, $\tilde{\cO} \left( H_2^{[M]} \right) = \tilde{\cO} \left( H_1^{[M]} \right)$). This result, derived in Section \ref{sec:multi}, builds upon the SAR strategy introduced in Section \ref{sec:mbest}. The improvement with respect to \cite{GGLB11} is that our strategy is parameter-free, while the theoretical Gap-E introduced in \cite{GGLB11} requires the knowledge of $H_1^{[M]}$ to tune its parameter. Moreover the analysis of SAR is much simpler than the one of Gap-E. 
\newline

For each arm $i$ and all time rounds $t \geq 1$, we denote by $T_{i}(t)=\sum_{s=1}^t \mathds{1}_{I_t = i}$ the number of times arm $i$ was pulled from rounds $1$ to $t$, and by $X_{i,1}, X_{i,2}, \ldots, X_{i,T_{i,t}}$ the sequence of associated rewards. Introduce $\wh{\mu}_{i,s}=\frac{1}{s} \sum_{t=1}^s X_{i,t}$ the empirical mean of arm $i$ after $s$ evaluations. Denote by $X_{i,s}(m)$ and $\wh{\mu}_{i,s}(m)$ the corresponding quantities in the multi-bandit problem.

\section{$m$-best arms identification} \label{sec:mbest}
In this section we describe and analyze a new algorithm, called SAR (Sucessive Accepts and Rejects), for the $m$-best arms identification problem, see Figure \ref{fig:SAR} for its precise description. The idea behind SAR is similar to the one for SR (Successive Rejects) that was designed for the (single) best arm identification problem, with the additional feature that SAR sometimes \emph{accepts} an arm because it is confident enough that this arm is among the $m$ top arms. Informally SAR proceeds as follows. First the algorithm divides the time (i.e., the $n$ rounds) in $K-1$ phases. At the end of each phase, the algorithm either accepts the arm with the highest empirical mean or dismisses the arm with the lowest empirical mean, and in both cases the corresponding arm is deactivated. During the next phase, it pulls equally often each active arm. The key to decide whether to accept or reject during a certain phase $k$ is to rely on estimates for the gaps $\Delta_i^{\langle m \rangle}$. More precisely, assume that the algorithm has already accepted $m-m(k)$ arms $J_1, \hdots, J_{m-m(k)}$, i.e. there is $m(k)$ arms left to find. Then, at the end of phase $k$, SAR computes for the $m(k)$ empirical best arms (among the active arms) the distance (in terms of empirical mean) to the $(m(k)+1)^{th}$ empirical best arm among the active arms. On the other hand for the active arms that are not among the $m(k)$ empirical best arms, SAR computes the distance to the $m(k)^{th}$ empirical best arm. Finally SAR deactivates the arm $i_k$ that maximizes these empirical distances. If $i_k$ is currently the empirical best arm, then SAR accepts $i_k$ and sets $m(k+1) = m(k) - 1$, $J_{m-m(k+1)} = i_k$, and otherwise it simply rejects $i_k$. The length of the phases are chosen similarly to what was done for the SR algorithm.

\begin{figure}[t]
\bookbox{

\medskip\noindent
Let $A_1=\{1,\hdots,K\}$, $m(1) = m$, $\LG = \frac{1}{2} + \sum_{i=2}^K \frac{1}{i}$, $n_0=0$ and for $k \in \{1,\hdots, K-1\}$,
  $$
  n_k=\bigg\lceil \frac{1}{\LG} \frac{n-K}{K+1-k} \bigg\rceil.
  $$

\medskip\noindent
For each phase $k=1,2,\ldots,K-1$:
\begin{itemize}
\item[(1)]
For each active arm $i \in A_k$, select arm $i$ for $n_k - n_{k-1}$ rounds.
\item[(2)]
Let $\sigma_k : \{1,\ldots,K+1-k\} \to A_k$ be the bijection that orders the empirical means by $\wh\mu_{\sigma_k(1),n_k} \geq  \wh\mu_{\sigma_k(2),n_k} \geq \cdots \geq \wh\mu_{\sigma_k(K+1-k),n_k}$. For $1\leq r\leq K+1-k$, define empirical gaps
$$
	\wh\Delta_{\sigma_k(r), n_k} = \begin{cases}
		\wh\mu_{\sigma_k(r), n_k} - \wh\mu_{\sigma_k(m(k)+1), n_k} & \text{ if $r \leq m(k)$}\\
		\wh\mu_{\sigma_k(m(k)), n_k} - \wh\mu_{\sigma_k(r), n_k} & \text{ if $r \geq m(k)+1$}\\		
	\end{cases}
$$
\item[(3)]
Let $i_k \in \argmax_{i \in A_k} \wh\Delta_{i, n_k}$ (ties broken arbitrarily). Deactivate arm $i_k$, that is set $A_{k+1} = A_k \setminus \{i_k \}$.
\item[(4)]
If $\wh{\mu}_{i_k, n_k} > \wh{\mu}_{\sigma_k(m(k)+1), n_k}$ then arm $i_k$ is accepted, that is set $m(k+1) = m(k) - 1$ and $J_{m-m(k+1)} = i_k$.
\end{itemize}

\medskip\noindent
Output: The $m$ accepted arms $J_1, \hdots, J_m$.
}
\caption{\label{fig:SAR}
SAR (Successive Accepts and Rejects) algorithm for $m$-best arms identification.}
\end{figure}

\begin{theorem} \label{th:SAR}
The probability of error of SAR in the $m$-best arms identification problem satisfies
  $$e_n \leq 2 K^2 \exp\left(- \frac{n - K}{8 \LG H_2^{\langle m \rangle}} \right).$$
\end{theorem}

\begin{proof}
Consider the event $\xi$ defined by
$$
	\xi = \left\{\forall i \in \{1, \hdots, K\}, k \in \{1, \hdots, K-1\}, \left|\frac{1}{n_k}\sum_{s=1}^{n_k} X_{i,s} - \mu_i \right| \leq \frac{1}{4} \Delta_{(K+1-k)}^{\langle m \rangle}\right\}.
$$
By Hoeffding's Inequality and an union bound, the probability of the complementary event $\bar\xi$ can be bounded as follows
\begin{align*}
	\mathbb{P}(\bar\xi) & \leq \sum_{i=1}^K \sum_{k=1}^{K-1} \mathbb{P}\left(\left|\frac{1}{n_k}\sum_{s=1}^{n_k} X_{i,s} - \mu_i \right| > \frac{1}{4} \Delta_{(K+1-k)}^{\langle m \rangle}\right)\\
	& \leq \sum_{i=1}^K \sum_{k=1}^{K-1} 2\exp(-2n_k(\Delta_{(K+1-k)}^{\langle m \rangle}/4)^2) \\
	& \leq 2 K^2 \exp\left(- \frac{n - K}{8 \LG H_2^{\langle m \rangle}} \right),
\end{align*}
where the last inequality comes from the fact that
$$
	n_k \left(\Delta_{(K+1-k)}^{\langle m \rangle}\right)^2 \geq \frac{n-K}{\LG (K+1-k) \left(\Delta_{(K+1-k)}^{\langle m \rangle}\right)^{-2}} \geq \frac{n-K}{\LG H_2^{\langle m \rangle}}.
$$
Thus, it suffices to show that on the event $\xi$, the algorithm does not make any error. We prove this by induction on $k$. Let $k\geq 1$. Assume the algorithm makes no error in all previous $k-1$ stages. Note that event $\xi$ implies that at the end of stage $k$, all empirical means are within $\frac{1}{4} \Delta_{(K+1-k)}^{\langle m \rangle}$ of the respective true means.

Let $A_k = \{a_1,\ldots,a_{K+1-k}\}$ be the the set of active arms during phase $k$. We order the $a_i$'s such that $\mu_{a_1} > \mu_{a_2} > \cdots > \mu_{a_{K+1-k}}$. To slightly lighten the notation we denote $m' = m(k)$ for the number of arms that are left to find in phase $k$. The assumption that no error occurs in the first $k-1$ stages implies that 
$$
	a_1, a_2, \ldots, a_{m'} \in \{1,\ldots,m\}, \quad a_{m'+1},\ldots,a_{K+1-k} \in \{m+1,\ldots, K\}.
$$
If an error is made at stage $k$, it can be one of the following two types:
\begin{enumerate}
	\item The algorithm accepts $a_j$ at stage $k$ for some $j\geq m'+1$.
	\item The algorithm rejects $a_j$ at stage $k$ for some $j\leq m'$.
\end{enumerate}
Again to slightly shorten the notation we denote $\sigma = \sigma_k$ for the bijection (from $\{1, \hdots, K+1-k\}$ to $A_k$) such that $\wh\mu_{\sigma(1),n_k} \geq  \wh\mu_{\sigma(2),n_k} \geq \cdots \geq \wh\mu_{\sigma(K+1-k),n_k}$.  Suppose Type 1 error occurs. Then $a_j = \sigma(1)$ since if the algorithm accepts, it must accept the empirical best arm. Furthermore we also have 
\begin{equation}
	\wh\mu_{a_j,n_k} - \wh\mu_{\sigma(m'+1),n_k} \geq \wh\mu_{\sigma(m'),n_k} - \wh\mu_{\sigma(K+1-k),n_k} ,
	\label{3.1}
\end{equation}
since otherwise the algorithm would rather reject arm $\sigma(K+1-k)$. The condition $a_j = \sigma(1)$ and the event $\xi$ implies that 
\begin{align*}
	\wh\mu_{a_j, n_k} \geq \wh\mu_{a_1, n_k} \quad &\Rightarrow \quad \mu_{a_j} + \frac{1}{4}\Delta_{(K+1-k)}^{\langle m \rangle} \geq \mu_{a_1} - \frac{1}{4}\Delta_{(K+1-k)}^{\langle m \rangle}\\
	& \Rightarrow \quad \Delta_{(K+1-k)}^{\langle m \rangle} > \frac{1}{2}\Delta_{(K+1-k)}^{\langle m \rangle} \geq \mu_{a_1} - \mu_{a_j} \geq \mu_{a_1} - \mu_{m+1}
\end{align*}
We then look at the condition \eqref{3.1}. In the event of $\xi$, for all $i\leq m'$, we have 
$$
	\wh\mu_{a_i, n_k} \geq \mu_{a_i} - \frac{1}{4}\Delta_{(K+1-k)}^{\langle m \rangle} \geq \mu_{a_{m'}} - \frac{1}{4}\Delta_{(K+1-k)}^{\langle m \rangle} \geq \mu_m - \frac{1}{4}\Delta_{(K+1-k)}^{\langle m \rangle}. 
$$
So there are $m+1$ arms in $A_k$ (namely $a_1, a_2, \ldots, a_{m'}, a_j$) whose empirical means are at least $\mu_m - \frac{1}{4}\Delta_{(K+1-k)}^{\langle m \rangle}$, which means 
$
	\wh\mu_{\sigma(m'+1),n_k} \geq \mu_m- \frac{1}{4}\Delta_{(K+1-k)}^{\langle m \rangle}.
$
On the other hand, $\wh\mu_{\sigma(K+1-k),n_k} \leq \wh\mu_{a_{K+1-k}, n_k} \leq \mu_{a_{K+1-k}} + \frac{1}{4}\Delta_{(K+1-k)}^{\langle m \rangle}$. Therefore, using those two observations and \eqref{3.1} we deduce
\begin{align*}
	& \left(\mu_{a_j} + \frac{1}{4}\Delta_{(K+1-k)}^{\langle m \rangle}\right) - \left(\mu_m - \frac{1}{4}\Delta_{(K+1-k)}^{\langle m \rangle}\right) \\
	& \hspace{4cm} \geq \left (\mu_m - \frac{1}{4}\Delta_{(K+1-k)}^{\langle m \rangle}\right) - \left(\mu_{a_{K+1-k}} + \frac{1}{4}\Delta_{(K+1-k)}^{\langle m \rangle}\right) \\
	& \Rightarrow \quad \Delta_{(K+1-k)}^{\langle m \rangle} \geq 2\mu_m - \mu_{a_j} - \mu_{a_{K+1-k}} > \mu_m - \mu_{a_{K+1-k}}.
\end{align*}
Thus so far we proved that if there is a Type 1 error, then
$$\Delta_{(K+1-k)}^{\langle m \rangle} > \max(\mu_{a_1} - \mu_m, \mu_m - \mu_{a_{K+1-k}}) .$$
But at stage $k$, only $k-1$ arms have been accepted or rejected, thus $\Delta_{(K+1-k)}^{\langle m \rangle} \leq \max(\mu_{a_1} - \mu_m, \mu_m - \mu_{a_{K+1-k}})$. By contradiction, we conclude that Type 1 error does not occur.

Suppose Type 2 error occurs. The reasoning is symmetric to Type 1. In fact, if we rephrase the problem as finding the $K-m$ worst arms instead of the $m$ best arms, this is exactly the same as Type 1 error. Hence Type 2 error cannot occur as well. This completes the induction and consequently the proof of the theorem.
\end{proof}

\section{Multi-bandit best arm identification} \label{sec:multi}
In this section we use the idea of SAR for multi-bandit best arm identification. Here at the end of each phase we estimate the gaps $\Delta_i(m)$ within each problem, and we reject the arm with the largest such estimated gap. Moreover if a problem is left with only one active arm, then this arm is accepted and the problem is deactivated. The corresponding strategy is described precisely in Figure \ref{fig:SAR2}

\begin{figure}[t]
\bookbox{

\medskip\noindent
Let $A_1=\{(1,1), \hdots,(K,M) \}$, $\LGG = \frac12 + \sum_{i=2}^{M K} \frac{1}{i}$, $n_0=0$ and for $k \in \{1,\hdots, M K- 1\}$,
  $$
  n_k=\bigg\lceil \frac{1}{\LGG} \frac{n-M K}{M K+1-k} \bigg\rceil.
  $$

\medskip\noindent
For each phase $k=1,2,\ldots,MK-1$:
\begin{itemize}
\item[(1)]
For each active pair (arm, problem) $(i,m) \in A_k$, select arm $i$ in problem $m$ for $n_k - n_{k-1}$ rounds.
\item[(2)]
Let $h_k(m)$ be the arm with the highest empirical mean $\wh{\mu}_{i, n_k}(m)$ among the active arms in the active problem $m$ (that is such that $(i,m) \in A_k$).
\item[(3)] If there is a problem $m$ such that $h_k(m)$ is the last active arm in problem $m$, then deactivate both the arm and the problem, and accept the arm. That is, set $A_{k+1} = A_k \setminus \{(h_k(m), m)\}$ and $J_{m} = h_k(m)$. Otherwise proceed to step (4).
\item[(4)]
Let $(i_k,m_k) \in \argmax_{(i,m) \in A_k} \left(\wh{\mu}_{h_k(m), n_k}(m) - \wh{\mu}_{i, n_k}(m)\right)$ (ties broken arbitrarily). Deactivate arm $i_k$ in problem $m_k$, that is set $A_{k+1} = A_k \setminus \{(i_k, m_k)\}$.
\end{itemize}

\medskip\noindent
Output: The $M$ accepted arms $(J_1,1), \hdots, (J_M,M)$ (where the last accepted arm is defined by the unique element of $A_{M K}$).
}
\caption{\label{fig:SAR2}
SAR (Successive Accepts and Rejects) algorithm for the multi-bandit best arm identification.}
\end{figure}

\begin{theorem} \label{th:SAR2}
The probability of error of SAR in the multi-bandit best arm identification problem satisfies
  $$e_n \leq 2M^2K^2\exp\left(-\frac{n-MK}{8\overline{\log}(MK)H_2^{[M]}}\right).$$
\end{theorem}

\begin{proof}
Consider the event $\xi$ defined by
\begin{align*}
	\xi = \biggl\{
		\forall \ 1\leq i\leq K, & \ 1\leq m\leq M, \ 1\leq k \leq M K -1\\
		&\left|\frac{1}{n_k} \sum_{s=1}^{n_k}X_{i,s}(m) - \mu_{i}(m)\right| \leq \frac{1}{4}\Delta_{(MK+1-k)}
	\biggr\}.
\end{align*}
Following the same reasoning than in the proof of Theorem \ref{th:SAR}, it suffices to show that in the event of $\xi$ the algorithm makes no error. We do this by induction on the phase $k$ of the algorithm. Let $k\geq 1$. Assume the algorithm makes no error in all previous $k-1$ stages. Then at phase $k$, for all active problem $m$, the arm $(1,m)$ is still active. Moreover, as only $k-1$ arms have been deactivated, one clearly has
$$
	\max_{(i,m)\in A_k} (\mu_1(m) - \mu_i(m)) \geq \Delta_{(MK+1-k)}.
$$
Suppose the above maximum is achieved for the arm $(i^*, m^*)$, so we have
\begin{equation}
\label{4.1}
	\mu_1(m^*) - \mu_{i^*}(m^*) \geq \Delta_{(MK+1-k)}.
\end{equation}
Assume now that the algorithm makes an error at the end of phase $k$, i.e. some arm $(1,m)$ is deactivated and it was not the last active arm in problem $m$. For this to happen, we necessarily have for some $j \in \{2, \hdots, K\}$ (e.g., $j = h_k(m)$),
\begin{equation}
	\wh \mu_{j,n_k}(m) - \wh \mu_{1,n_k}(m) \geq \wh \mu_{1, n_k}(m^*) - \wh \mu_{i^*, n_k}(m^*). \label{4.2}
\end{equation}
Clearly on the event $\xi$ one has
\begin{align*}
&	\wh \mu_{j,n_k}(m) - \wh \mu_{1,n_k}(m) \\
& =  \wh \mu_{j,n_k}(m) - \mu_j(m) + \mu_j(m) - \mu_1(m) +  \mu_1(m) - \wh \mu_{1,n_k}(m) \\
	 & < \frac{1}{2}\Delta_{(MK+1-k)}.
\end{align*}
On the other hand, using \eqref{4.1} and $\xi$, one has
\begin{align*}
	&\wh \mu_{1,n_k}(m^*) - \wh \mu_{i^*,n_k}(m^*) \\
&= \wh \mu_{1,n_k}(m^*) - \mu_{1}(m^*) + \mu_{1}(m^*) - \mu_{i^*}(m^*) + \mu_{i^*}(m^*) - \wh \mu_{i^*,n_k}(m^*) \\
	 &\geq \frac{1}{2}\Delta_{(MK+1-k)}.
\end{align*}
Therefore, $\wh \mu_{1,n_k}(m^*) - \wh \mu_{i^*,n_k}(m^*) > \wh \mu_{j,n_k}(m) - \wh \mu_{1,n_k}(m)$, contradicting \eqref{4.2}. This completes the induction and the proof.
\end{proof}

\section{Experiments} \label{sec:exp}
In this section we revisit the simple experiments of \cite{ABM10} in the setting of multiple identifications. Since our objective is simply to illustrate our theoretical analysis we focus on the $m$-best arms identification problem, but similar numerical simulations could be conducted in the multi-bandit setting and compared to the results of \cite{GGLB11}.

We compare our proposed strategy SAR to three competitors: The uniform sampling strategy that divides evenly the allocation budget $n$ between the $K$ arms, and then return the $m$ arms with the highest empirical mean (see \cite{BMS11} for a discussion of this strategy in the single best arm identification). The SR strategy is the plain Successive Rejects strategy of \cite{ABM10} which was designed to find the (single) best arm. We slightly improve it for $m$-best identification by running only $K-m-1$ phases (while still using the full budget $n$) and then returning the last $m$ surviving arms. Finally we consider the extension of UCB-E to the $m$-best arms identification problem, which is based on a similar idea than the extension Gap-E of \cite{GGLB11} for the multi-bandit best arm identification, see Figure \ref{fig:GapE} for the details. Note that this last algorithm requires to know the complexity $H_1^{\langle m \rangle}$. One could propose an adaptive version, using ideas described in \cite{ABM10}, but for sake of simplicity we restrict our attention to the non-adaptive algorithm.

\begin{figure}[t] 
\bookbox{\small
Parameter: exploration parameter $c>0$.

\medskip\noindent
For each round $t=1,2,\ldots,n$:
\begin{itemize}
\item[(1)]
Let $\sigma_t$ be the permutation of $\{1, \hdots, K\}$ that orders the empirical means, i.e., $\wh\mu_{\sigma_t(1),T_{\sigma_t(1)}(t-1)} \geq \wh\mu_{\sigma_t(2),T_{\sigma_t(2)}(t-1)} \geq \cdots \geq \wh\mu_{\sigma_t(K),T_{\sigma_t(K)}(t-1)}$. For $1\leq r \leq K$, define the empirical gaps
$$
	\wh\Delta_{\sigma_t(r), t} = \begin{cases}
		\wh\mu_{\sigma_t(r), T_{\sigma_t(r)}(t-1)}  - \wh\mu_{\sigma_t(m+1), T_{\sigma_t(m+1)}(t-1)}  & \text{ if $r \leq m$}\\
		\wh\mu_{\sigma_t(m), T_{\sigma_t(m)}(t-1)} - \wh\mu_{\sigma_t(r), T_{\sigma_t(r)}(t-1)} & \text{ if $r \geq m+1$}\\		
	\end{cases}
$$
\item[(2)]
Draw $$I_t \in \argmax_{i \in \{1,\hdots,K\}} - \wh \Delta_{i, t} + c \sqrt{\frac{n / H_1^{\langle m \rangle}}{T_i(t-1)}} .$$
\end{itemize}

\medskip\noindent
Let $J_1, \hdots, J_m$ be the $m$ arms with highest empirical means $\wh{\mu}_{i,T_i(n)}$.
}
\caption{\label{fig:GapE}
Gap-E algorithm for the $m$-best arms identification problem.}
\end{figure}

In our experiments we consider only Bernoulli distributions, and the optimal arm always has parameter $1/2$. Each experiment corresponds to a different situation for the gaps, they are either clustered in few groups, or distributed according to an arithmetic or geometric progression. For each experiment we plot the probability of misidentification for each strategy, varying $m$ between $2$ and $K-1$. The allocation budget for each experiment is chosen to be roughly equal to $\max_{1 \leq m \leq K-1} H_1^{\langle m \rangle}$. We report our results in Figure \ref{fig:exp}. The parameters for the experiments are as follows:
\begin{itemize}
	\item Experiment 1: One group of bad arms, $K=20$, $\mu_{2:20} = 0.4$ (meaning for any $j\in\{2,\dots,20\}, \mu_j=0.4$)
	\item Experiment 2: Two groups of bad arms, $K=20$, $\mu_{2:6} = 0.42$, $\mu_{7:20}=0.38$.
	\item Experiment 3: Geometric progression, $K=4$, $\mu_{i} = 0.5 - (0.37)^i$, $i \in \{2,3,4\}$.
	\item Experiment 4: $6$ arms divided in three groups, $K=6$, $\mu_{2} = 0.42$, $\mu_{3:4}=0.4$, $\mu_{5:6}=0.35$.
	\item Experiment 5: Arithmetic progression, $K=15$, $\mu_{i} = 0.5 - 0.025 i$, $i \in \{2,\hdots,15\}$.
	\item Experiment 6: Three groups of bad arms, $K=30$, $\mu_{2:6}=0.45$, $\mu_{7:20}=0.43$, $\mu_{21:30}=0.38$.
\end{itemize}
It is interesting to note that SR performs badly for $m$-best arms identification when $m>1$, as it has even worse performances than the naive uniform sampling in many cases. This shows that the tradeoffs involved in finding the single best arm and finding the top $m$ arms are fundamentally different. As expected SAR always outperforms uniform sampling, and Gap-E has slightly better performances than SAR (but Gap-E requires an extra information to tune its parameter, and the adapative version comes with no provable guarantee).

\begin{figure}[t]
\begin{center}
\includegraphics[scale = 0.52]{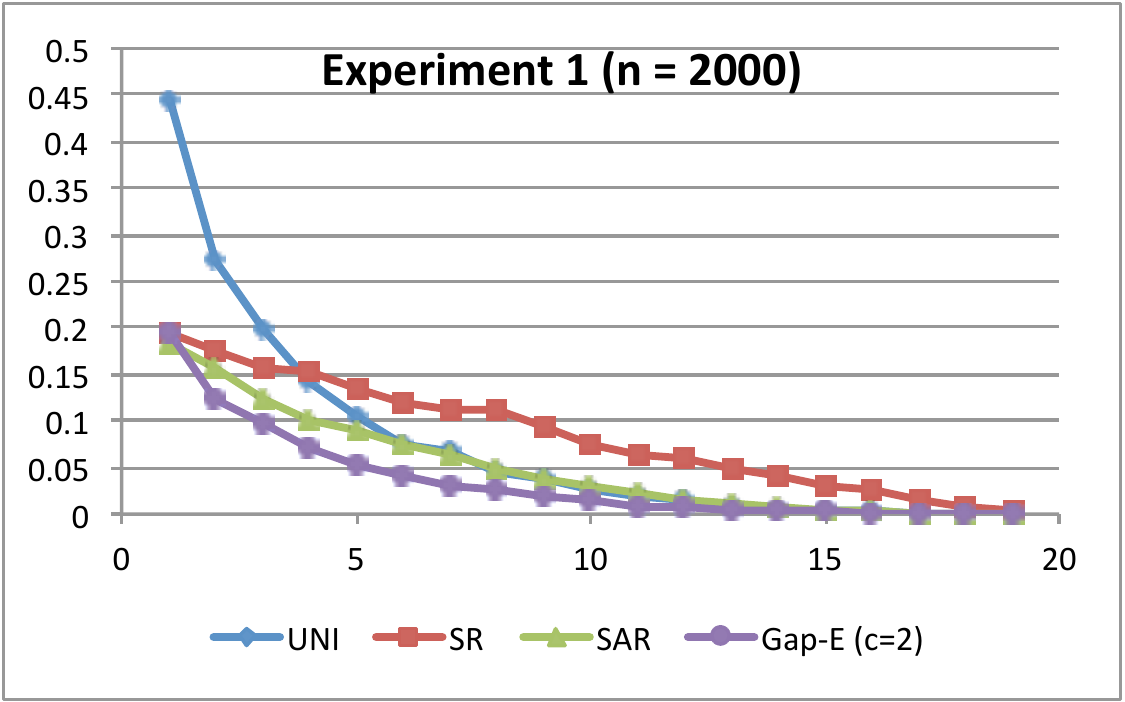}\hfill \includegraphics[scale = 0.52]{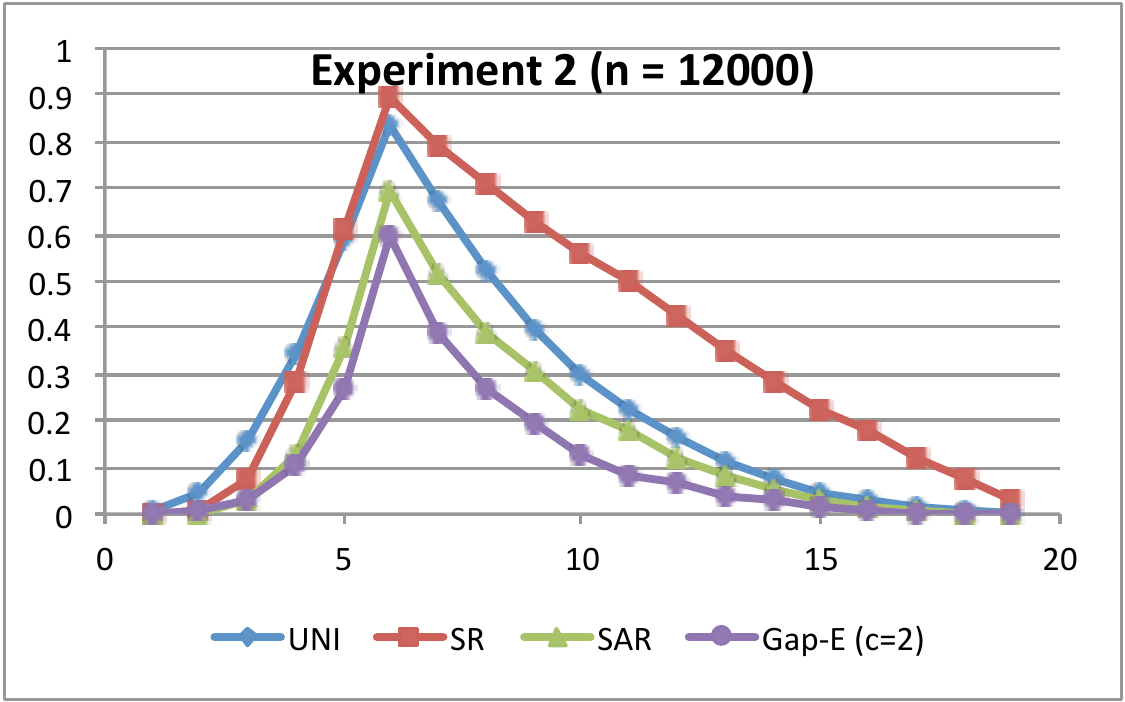}\\
\vspace{0.1in}
\includegraphics[scale = 0.52]{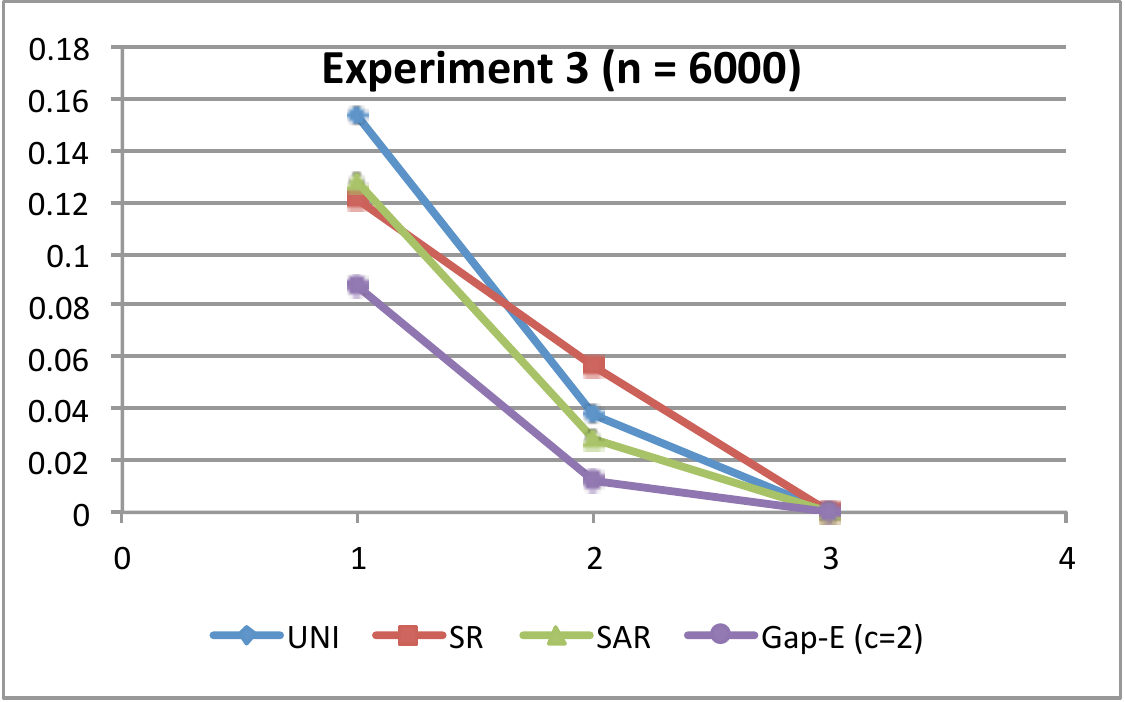}\hfill \includegraphics[scale = 0.52]{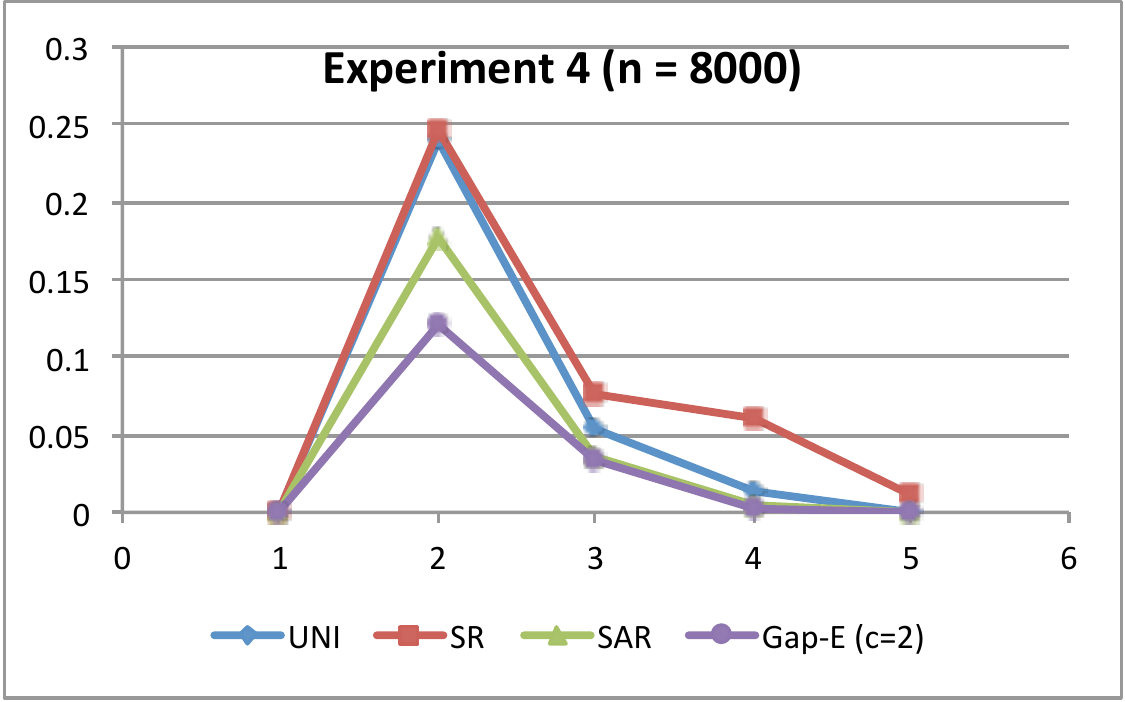}\\
\vspace{0.1in}
\includegraphics[scale = 0.52]{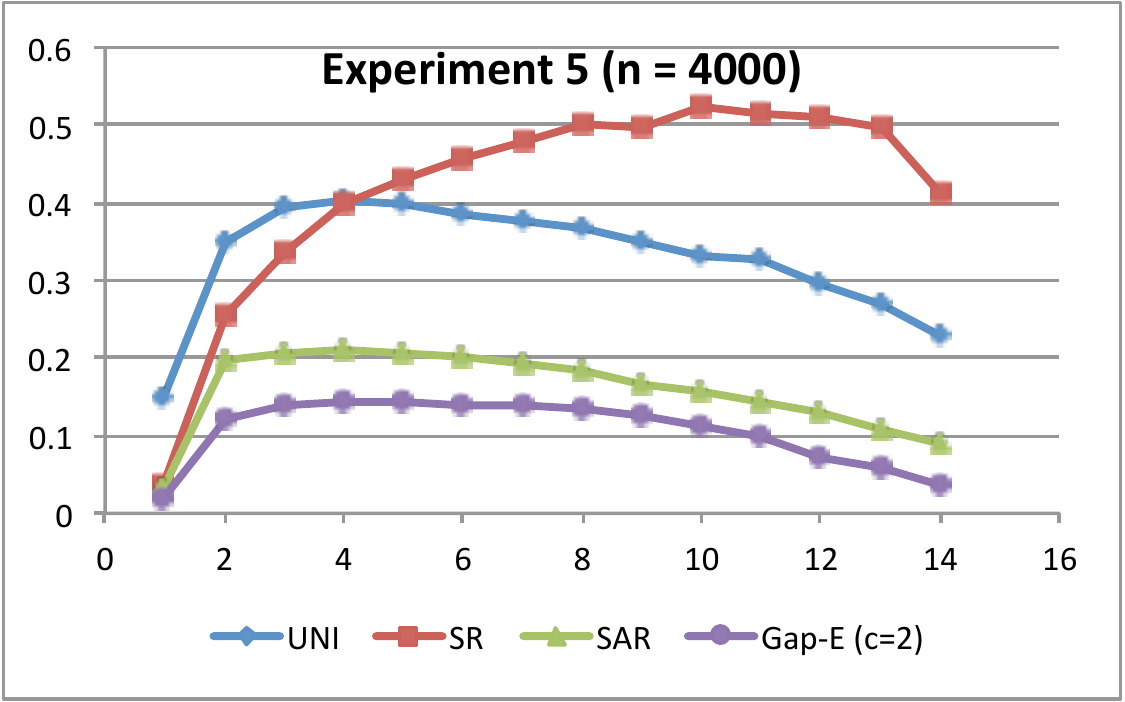}\hfill \includegraphics[scale = 0.52]{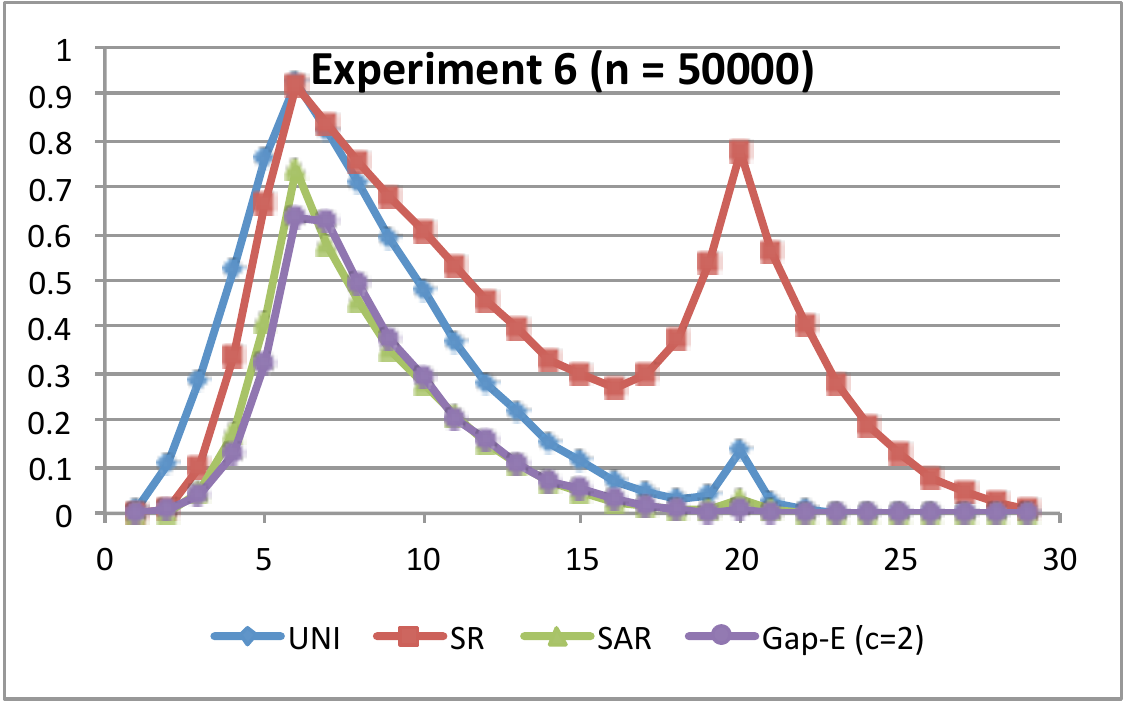}\\
\end{center}
\caption{Numerical simulations for the $m$-best arms identification problem. We chose $c = 2$ (exploration parameter) for the Gap-E algorithm in all experiments.}
\label{fig:exp}
\end{figure}

\bibliographystyle{plainnat}
\bibliography{newbib}

\end{document}